\documentclass[11pt, a4paper]{article}

\usepackage[a4paper]{geometry}
\usepackage[english]{babel}
\usepackage[utf8]{inputenc}
\usepackage{bm}
\usepackage{xspace}
\usepackage{amsmath,amsthm,amssymb,mathtools}
\usepackage{dsfont,bbm}

\usepackage{bbm}
\usepackage{url}
\usepackage[caption=false]{subfig}
\usepackage{natbib}
\usepackage{algorithm}
\usepackage[noend]{algpseudocode}
\usepackage{xcolor}
\usepackage{tikz}
\usepackage{graphicx}
\renewcommand{\labelenumi}{(\alph{enumi})}
\renewcommand\theenumi\labelenumi
\usepackage{diagbox}
\usepackage{tabularx}
\newcolumntype{Y}{>{\centering\arraybackslash}X}

\newcommand{\sdooea}{SD-(1+1)~EA\xspace}
\newcommand{\oneoneea}{\ooea}
\newcommand{\rls}{RLS\xspace}
\newcommand{\rlsonetwo}{RLS$^{1,2}$\xspace}
\newcommand{\sdrls}{SD-RLS$^{\text{p}}$\xspace}
\newcommand{\sdrlss}{SD-RLS$^{\text{r}}$\xspace}
\newcommand{\sdrlsss}{SD-RLS$^{\text{m}}$\xspace}

\newcommand{\oofea}{(1+1)~FEA$_\beta$\xspace}
\newcommand{\ooea}{(1+1)~EA\xspace}

\newcommand{\opllga}{(1+($\lambda,\lambda$))~GA\xspace}
\newcommand{\om}{\textsc{OneMax}\xspace}

\newcommand{\onemax}{\om}
\newcommand{\jump}{\textsc{Jump}\xspace}

\newcommand{\jumpo}{\textsc{JumpOff}\xspace}

\newcommand{\needhighmut}{\textsc{NeedHighMut}\xspace}
\newcommand{\preferonebit}{\textsc{PreferOneBitFlip}\xspace}
\newcommand{\needglobalmut}{\textsc{NeedGlobalMut}\xspace}
\newcommand{\pre}{\textsc{pre}}
\newcommand{\suff}{\textsc{suff}}

\newcommand{\R}{\ensuremath{\mathbb{R}}}
\newcommand{\N}{\ensuremath{\mathbb{N}}}

\DeclareMathOperator{\Prob}{Pr}

\newcommand{\ones}[1]{\lvert #1\rvert_1}

\newcommand{\ie}{i.\,e.\xspace}
\newcommand{\eg}{e.\,g.\xspace}

\newcommand{\card}[1]{\lvert #1\rvert}

\DeclareMathOperator{\gap}{gap}

\DeclareMathOperator{\opt}{opt}
\DeclareMathOperator{\im}{Im}

\newcommand{\xopt}{x_{\opt}}
\newcommand{\prob}[1]{\mathord{\Prob}\mathord{\left(#1\right)}}

\newcommand{\expect}[1]{\mathord{\mathrm{E}}\mathord{\left(#1\right)}}

\newcommand{\TG}{TG\xspace}
\newcommand{\erdosrenyi}{Erdős–Rényi\xspace}
\newcommand{\Kn}{$K_n$\xspace}

\newenvironment{proofof}[1]{\begin{proof}[Proof of~#1]}{\end{proof}}

\algnewcommand{\IfThenElse}[3]{
  \State \algorithmicif\ #1\ \algorithmicthen\ #2\ \algorithmicelse\ #3}

\algnewcommand{\IfThen}[2]{
  \State \algorithmicif\ #1\ \algorithmicthen\ #2}

\DeclarePairedDelimiter\floor{\lfloor}{\rfloor}

\newtheorem{lemma}{Lemma}
\newtheorem{theorem}{Theorem}

\newtheorem{corollary}{Corollary}

\begin{document}

\title{Stagnation Detection in \\ Highly Multimodal Fitness Landscapes}
\author{
  Amirhossein Rajabi\\
  Technical University of Denmark \\
	Kgs. Lyngby \\
	Denmark \\
amraj@dtu.dk \\
  \and
    Carsten Witt\\
  Technical University of Denmark \\
	Kgs. Lyngby \\
	Denmark \\
cawi@dtu.dk \\
}

\maketitle

\begin{abstract}
Stagnation detection has been proposed as a mechanism for randomized search 
heuristics to escape from local optima by automatically increasing the size of the neighborhood to find the so-called 
gap size, \ie, the distance to the next improvement. 
Its usefulness has mostly been considered in simple multimodal landscapes with few local optima that could be crossed one after another. In multimodal landscapes with a more complex location of optima of similar gap size, stagnation detection suffers from the fact that the neighborhood size is frequently reset to~$1$ without 
using gap sizes that were promising in the past.

In this paper, we investigate a new mechanism called \emph{radius memory} which can be added to stagnation detection to 
control the search radius  more carefully 
 by giving preference to values that were successful in the past.
We implement this idea in an algorithm called \sdrlsss and show compared to previous variants of stagnation detection that it yields 
speed-ups for linear functions under uniform constraints and 
the minimum spanning 
tree problem. Moreover, its running time does 
not significantly deteriorate on unimodal functions and a 
generalization of the 
\jump benchmark.
Finally, we present experimental results carried out to study \sdrlsss and compare it with other algorithms.
\end{abstract}

\section{Introduction}
The theory of self-adjusting evolutionary algorithms (EAs) is a research area that 
has made significant progress in recent years \citep{DoerrDoerrParameterBookChapter}. 
For example, a self-adjusting choice of mutation and crossover probability in the 
so-called \opllga allows an expected optimization time of $O(n)$ on \om, which is 
not possible with any static setting \citep{DoerrEbelTCS15,DoerrDoerrAlgorithmica18}. Many studies 
focus on unimodal problems, while self-adjusting EAs for multimodal problems 
in discrete search spaces, more precisely for pseudo-Boolean optimization,  
were investigated only recently  from a theoretical 
runtime perspective. \emph{Stagnation detection} proposed in \cite{RajabiWittGECCO20} addresses a shortcoming of classical 
self-adjusting EAs, 
which try to learn promising parameter settings from fitness
improvements. 
Despite the absence of a fitness improvements 
when the best-so-far solution is at a local optimum,
stagnation detection learns from the number of 
unsuccessful steps and adjusts the mutation rate if this number exceeds 
a certain threshold. Thanks to this mechanism, the so-called \sdooea 
proposed in \cite{RajabiWittGECCO20} optimizes the classical \jump 
function with~$n$ bits and gap size~$m$ in 
expected time $O((en/m)^m)$, which 
corresponds asymptotically to the best possible time achievable 
through standard bit mutation, more precisely when each bit is 
flipped independently with probability~$m/n$. It is worth pointing 
out that stagnation detection does not have any prior information about 
the gap size~$m$.

Although leaving a local optimum requires a certain number of bits
 to be flipped simultaneously, which we call the \emph{gap size}, 
 the \sdooea mentioned above still 
 performs independent bit flips. Therefore, even 
 for the best setting of the mutation rate, 
 only the expected number of 
 flipping bits equals the gap size while the actual number of flipping 
 bits may be different. This has motivated \citet{RajabiWittEvo21} to consider the $k$-bit flip operator 
 flipping a uniform random subset of~$k$ bits 
 as known from randomized local search (RLS) \citep{DoerrDoerrAlgorithmica18} 
 and to adjust~$k$ via stagnation detection. Compared to the \sdooea, 
 this  allows a speed-up of $(\frac{ne}{m})^m/\binom{n}{m}$ (up to lower-order terms) on functions with 
 gap size~$m$ and a speed-up of up to roughly~$e=2.718\dots$ on unimodal
 functions while still being able to search globally.
 
 \citet{RajabiWittEvo21} emphasize that their RLS with 
 self-adjusting $k$-bit flip operator resembles variable neighborhood search 
 \citep{HansenMladenovic18} 
 but features less determinism by drawing the~$k$ bits to be flipped uniformly
 at random instead of searching the neighborhood in a fixed order. 
 The random behavior still maintains many characteristics 
 of the original RLS, 
 including independent stochastic decisions which ease the runtime analysis.
 If the bit positions to be flipped follow a deterministic scheme as 
 in quasirandom EAs \citep{DoerrFW10}, dependencies complicate the analysis 
 and make it difficult to apply tools like drift analysis. However, 
 a drawback 
 of the randomness is that the independent, uniform choice of the set of 
 bits to be flipped leaves a positive probability of missing an 
 improvement within the given time a specific parameter value~$k$ is tried. 
 Therefore, the first RLS variant with stagnation detection proposed 
 in \cite{RajabiWittEvo21} and called SD-RLS there 
 has infinite expected runtime 
 in general, but is efficient with high probability, where the success probability of the algorithm
 is controlled via the threshold value for the number of 
 steps without fitness improvement that trigger a change of~$k$. We remark here that the problem with 
 infinite runtime is not existent with the independent bit flips as long 
 as each bit is flipped with probability in the open interval $(0,1)$.
 
 In this paper, we denote by \sdrls the simple SD-RLS just proposed. To 
 guarantee finite expected optimization time, 
 \citet{RajabiWittEvo21} introduce a second variant that  repeatedly 
 returns to lower so-called (mutation) strengths, \ie,  number of bits flipped, while the algorithm 
 is still waiting for an improvement. The largest neighborhood size (\ie, 
 number of bits flipped) is denoted as \emph{radius~$r$} and, in essence, the strength~$s$ 
 is decreased in a loop 
 from~$r$ to~$1$ before the radius is increased. 
 Interestingly, the additional time spent at exploring smaller strengths 
 in this loop, with the right choice  of phase lengths, contributes 
 only a lower-order term to the typical time that \sdrls has in the absence 
 of errors. The resulting algorithm (Algorithm~\ref{alg:sdrls_robust}) is 
 called SD-RLS$^*$ in \cite{RajabiWittEvo21}, but in this paper referred to by  
 \sdrlss, where the label~\emph{r} denotes \emph{robust}.
 
 As already explained in \cite{RajabiWittEvo21}, \sdrlss (and also the plain 
 \sdrls) return to strength~$1$ after every fitness improvement and try this strength 
 for a sufficiently large time to find an improving one-bit flip with 
 high probability. This behavior can be undesired on highly multimodal 
 landscapes where progress is typically only made via larger strengths.
 As an example, the minimum spanning tree (MST) problem as originally 
 considered for the \ooea and an RLS variant in \cite{NeumannW07} requires 
 two-bit flips to make progress in its crucial optimization phase. Both 
 theoretically and experimentally, Rajabi and Witt~\citep{RajabiWittEvo21} 
 observed that \sdrlss is less efficient than the RLS variant from 
 \cite{NeumannW07} since low, useless strengths (here~$1$) are tried for 
 a too long period of time. On the other hand, it can also be risky 
 exclusively to proceed with the strength that was last found to be working 
 if the fitness landscape becomes easier
 at some point and progress can again 
 be made using smaller strengths.
 
 In this paper, we address this trade-off between exploiting high strengths
 that were necessary in the past and again trying smaller strengths for a certain 
 amount of time. We propose a mechanism called \emph{radius memory} that 
 uses the last successful strength value to assign a reduced budget of iterations 
 to smaller strengths. This budget
 is often much less than the number stereotypically tried
 in \sdrlss after every fitness improvement. However, 
 the budget  must be balanced carefully to allow the 
 algorithm to adjust itself to gap sizes becoming smaller over the run 
 of the algorithm. Our choice of budget is based on the number of 
 iterations (which is the same as the number of fitness 
 evaluations) passed to find the latest improvement and assigns 
 the same combined amount 
 of time, divided by $\ln n$, to smaller strengths tried afterwards. This choice, encorporated 
 in our new algorithm \sdrlsss, basically limits the time spent at unsuccessful strengths 
 by less than the waiting time for an improvement with the last successful strength but is
 still big enough to adjust to smaller strength sufficiently quickly. On the one hand, it
 (up to lower-order terms) preserves the runtime bounds 
 on general unimodal function classes and jump functions 
 shown for \sdrlss in \cite{RajabiWittEvo21}.  
 On the other hand, it significantly reduces the 
 time for the strength to 
 return to larger values on two highly multimodal problems, 
 namely optimization of linear functions under uniform constraints 
 and the MST. Although these ideas are implemented in a simple RLS 
 maintaining one individual only, we implicitly consider stagnation detection 
 as a module that can be added to other algorithms as shown 
 in \cite{RajabiWittGECCO20} and very recently in \cite{DoerrZhengAAAI2021} for multi-objective optimization. 
 Concretely, we could also use the stagnation detection 
 with radius memory in population-based algorithms. 
 
 This paper is structured as follows. In Section~\ref{sec:preliminaries}, 
 we define the algorithms considered and collect some important technical
 lemmas. Section~\ref{sec:analysis} presents time bounds for the 
 new algorithm \sdrlsss to leave local optima and  
 applies these to obtain
 bounds on the expected optimization time on unimodal
 and jump functions. 
 Moreover, it includes in Lemma~\ref{lem:realradius} the crucial analysis 
 of the time for the strength to settle at 
 smaller values when an improvement is missed. 
 Thereafter, these results are used in Section~\ref{sec:speedups} to analyze 
 \sdrlsss on linear functions under uniform
 constraints and to show a linear-time speedup 
 compared to the \sdrlss algorithm in \cite{RajabiWittEvo21}. Section~\ref{sec:mst}
 shows that \sdrlsss optimizes MST instances on graphs 
 with~$m$ edges 
 in expected time at most $(1+o(1))m^2\ln m$, which 
 is by an asymptotic factor of~$2$ faster than 
 the bound for \rlsonetwo from \cite{NeumannW07}
 and represents, to the best of our knowledge, 
 the first asymptotically tight analysis of a 
 globally searching (1+1)-type algorithm on the 
 problem. In Section~\ref{sec:slowdown}, we present 
 an example where the radius memory is detrimental and leads 
 to exponential optimization time with probability 
 $1-o(1)$ while the original \sdrlss from 
 \cite{RajabiWittEvo21} is efficient. Section~\ref{sec:experiments} presents experimental
 supplements to the analysis of  \sdrlss 
 and \sdrlsss and comparisons with other algorithms 
 from the literature
 before we finish with some conclusions.
 
\section{Preliminaries}
\label{sec:preliminaries}

\subsection{Algorithms}
We start by describing a class of
classical RLS algorithms 
and the considered 
extensions with stagnation detection. 
Algorithm~\ref{alg:rls-classic} is a simple 
hill climber  that uses a static 
strength~$s$ and always flips~$s$ bits uniformly 
at random. The special case where~$s=1$, \ie, 
using one-bit flips, has been  
investigated thoroughly in the literature 
\cite{DoerrDoerrAlgo16} and 
is mostly just called RLS.

\begin{algorithm}[htb]
	\caption{\rls with static strength~$s$ for the maximization of 
	$f\colon\{0,1\}^n \to \mathbb{R}$}
	\label{alg:rls-classic}
	\begin{algorithmic}
		\State Select $x$ uniformly at random from $\{0, 1\}^n$
		\For{$t \gets 1, 2, \dots$}
		\State Create $y$ by flipping $s$ bit(s) in a copy of $x$.
		\IfThen{$f(y) \ge f(x)$}{$x \gets y$.}
		\EndFor
	\end{algorithmic}
\end{algorithm}

The algorithm \rlsonetwo (which also 
is just called RLS in \cite{NeumannW07})
is an extension of this classical RLS (\ie, Algorithm~\ref{alg:rls-classic} with strength~$1$) choosing  
strength~$s\in\{1,2\}$ uniformly before flipping~$s$ bits. This 
extension is crucial for making progress on the MST problem.

In \cite{RajabiWittEvo21}, \rls is enhanced by stagnation
detection, leading to Algorithm~\ref{alg:sdrls_robust}. 
In a nutshell, the algorithm increases its strength 
after a certain number of unsuccessful steps 
according to the threshold value $\binom{n}{s}\ln R$ which has been chosen 
to bound the so-called failure probability at strength~$s$, \ie, 
the probability of not finding an improvement at Hamming distance~$s$, by 
at most $\bigl(1-1/\binom{n}{s}\bigr)^{\binom{n}{s}\ln R}\le 1/R$.
It also incorporates logic to return to smaller strengths repeatedly by maintaining the so-called 
radius value~$r$. All variables and parameters will be discussed 
in detail below when we come to our extension with 
radius memory. Algorithm~\ref{alg:sdrls_robust} 
is called SD-RLS$^{*}$ in \cite{RajabiWittEvo21} 
since that paper also discusses a simpler variant 
called SD-RLS without the logic related to the 
radius variable. However, that variant 
is not robust and has infinite expected 
optimization time in general even on unimodal problems.
We call Algorithm~\ref{alg:sdrls_robust} \sdrlss 
since, as argued in \cite{RajabiWittEvo21}, the 
radius makes the algorithm \emph{robust}.

\begin{algorithm}[htb]
	\caption{RLS with robust stagnation detection (\sdrlss) for the maximization of 
	$f\colon\{0,1\}^n \to \mathbb{R}$}
	\label{alg:sdrls_robust}
	\begin{algorithmic}
		\State Select $x$ uniformly at random from $\{0, 1\}^n$.
		\State $r \gets 1$, $s \gets 1$, $u\gets 0$.
		\For{$t \gets 1, 2, \dots$}
		\State Create $y$ by flipping $s$ bits in a copy of $x$ uniformly at random.
		\State $u\gets u+1$.
		\If{$f(y) > f(x)$}
		\State $x \gets y$.
		\State $r\gets 1$, $s\gets 1$, $u\gets 0$.
		\ElsIf {$f(y) = f(x)$ \textbf{and} $r=1$}
		\State $x \gets y$.

		\EndIf
		\If{$u > \binom{n}{s}\ln R$}
		\If{$s=1$}
		    \IfThenElse {$r<n/2$}{$r\gets r+1$}{$r\gets n$}
		    \State $s\gets r.$
		    \Else
		    \State $s\gets s-1$.
		\EndIf
		\State $u\gets 0$.
		\EndIf
		\EndFor
	\end{algorithmic}
\end{algorithm}

In the following, we present in 
Algorithm~\ref{alg:sdrls_star} the new 
algorithm \sdrlsss using 
stagnation detection and radius memory. It extends 
\sdrlss by 
adding logic for setting the helper variable~$B$ and by 
minimizing the original threshold $\binom{n}{s}\ln R$ 
for the number of unsuccessful steps with $\frac{u}{(\ln n)(r-1)}$ via 
the variable. Another 
minor change is that it increases the strength 
from~$1$ to the radius~$r$ instead of decreasing it for the sake of simplicity of the new algorithm and the proofs. We 
describe the algorithm in more detail now.

After a strict improvement with strength~$s$ (which becomes the initial radius~$r$ for the next search point), 
the algorithm uses all strengths~$s^\prime<r$ for $\min\{\frac{u}{(\ln n)(r-1)}, \binom{n}{s^\prime}\ln R\}$ attempts, 
where $u$ is the value of the counter at the time that the previous improvement happened. Once the current strength becomes 
equal to the current radius,
the threshold becomes $\min\{\infty, \binom{n}{s^\prime}\ln R\}=\binom{n}{s^\prime}\ln R$ for the rest of iterations with the current search point. Therefore, the cap at $\frac{u}{(\ln n)(r-1)}$ is only effective as long as the current radius 
equals~$r$ and the current strengths are smaller than~$r$. 
For technical reasons, the radius increases directly to~$n$ when 
it has passed $n/2$. Moreover, as another technical detail, we accept search 
points of equal fitness only if the current radius is one (leading to 
the same acceptance behavior as in classical RLS, see Algorithm~\ref{alg:rls-classic}), whereas only strict 
improvements are accepted at larger radii.

The factor $1/\ln n$ appearing in the first argument of the minimum 
\[\min\{\frac{u}{(\ln n)(r-1)}, \binom{n}{s^\prime}\ln R\}\] is a parameter choice that has turned out robust and useful 
in our analyses. The choice $\frac{u}{r-1}$, \ie, an implicit constant of~$1$, could seem more natural here since then the 
algorithm would look at smaller strengths as often as the last successful strength was tried; however, this would make
our forthcoming bounds worse by a constant factor.

As mentioned above, stagnation detection has also a parameter~$R$ to bound the probability of failing to find an improvement at the ``right'' strength. We formally prove in Lemma~\ref{lem:failure-probability} that the probability of not finding an improvement where there is a potential of making progress is at most~$1/R$. The recommendation of~$R$ for \sdrlss in \cite{RajabiWittEvo21} is still valid for \sdrlsss, \ie,  $R\ge n^{3+\epsilon}\cdot \card{\im f}$ for a constant~$\epsilon$ (where $\im f$ is the 
image set of~$f$), resulting in that the probability of ever missing an improvement at the right strength is sufficiently small throughout the run. However, in this paper, by improving some analyses, we recommend a tighter value for~$R$, namely $R\ge \max\{S, n^{3+\epsilon}\}$ for an arbitrary constant~$\epsilon>0$ where $S$ is an upper bound on the number of improvements during the run. Obviously, we can always choose $S = \card{\im f}$.

\begin{algorithm}[htb]
	\caption{RLS with robust stagnation detection and radius memory mechanism (\sdrlsss) for the maximization of 
	$f\colon\{0,1\}^n \to \mathbb{R}$ }
	\label{alg:sdrls_star}
	\begin{algorithmic}
		\State Select $x$ uniformly at random from $\{0, 1\}^n$.
		\State $r \gets 1$, $s \gets 1$, $u\gets 0$, $B\gets \infty$.
		\For{$t \gets 1, 2, \dots$}
		\State Create $y$ by flipping $s$ bits in a copy of $x$ uniformly at random.
		\State $u\gets u+1$.
		\If{$f(y) > f(x)$}
		\State $x \gets y$.
		\State $r\gets s$, $s\gets 1$.
		\IfThenElse{$r>1$}{$B\gets \frac{u}{(\ln n)(r-1)}$}{$B\gets \infty$.}
		\State $u\gets 0$.
		\ElsIf {$f(y) = f(x)$ \textbf{and} $r=1$}
		\State $x \gets y$.

		\EndIf
		\If{$u > \min\{B,\binom{n}{s}\ln R$\}}
		\If{$s=r$}
		    \IfThenElse {$r<n/2$}{$r\gets r+1$}{$r\gets n$}
		    \State $s\gets 1.$
		    \Else
		    \State $s\gets s+1$.
		    \IfThen{$s=r$}{$B\gets \infty$}
		\EndIf
		\State $u\gets 0$.
		\EndIf
		\EndFor
	\end{algorithmic}
\end{algorithm}

The \emph{runtime} or the \emph{optimization time} of a search heuristic on a function~$f$ 
 is the first point time~$t$ where a search point of optimal fitness has been created; often  the expected runtime, \ie, the expected value of this time, is analyzed.

\subsection{Mathematical tools}
\label{sec:tools}
In the following lemma, which has been taken from~\cite{RajabiWittEvo21}, we have some combinatorial inequalities that will be used in the analyses of the algorithms. 
 The part \ref{lem:partial-sum:sum} in Lemma~\ref{lem:partial-sum} seems to be well known and has already been proved in \cite{MathoverflowML17} and is also a consequence of Lemma~1.10.38 
in \cite{DoerrProbabilisticBookChapter}.  The part~\ref{lem:partial-sum:M} follows from elementary manipulations.

\begin{lemma}[Lemma~1 in \cite{RajabiWittEvo21}] \label{lem:partial-sum}
For any integer $m\le n/2$, we have 
	\begin{enumerate}
	    \item $\sum_{ i=1}^{m}\binom{n}{i}\leq \frac{n-(m-1)}{n-(2m-1)} \binom{n}{m}$, \label{lem:partial-sum:sum}
        \item $\binom{n}{M}\le \binom{n}{m}\left(\frac{n-m}{m}\right)^{M-m}$ for $m<M<n/2$. \label{lem:partial-sum:M}
	\end{enumerate}
\end{lemma}

\section{Analysis of the Algorithm \sdrlsss}
\label{sec:analysis}
In this section, we shall show bounds on the optimization time of  \sdrlsss in addition to useful technical  lemmas used in different analyses in the paper.

\subsection{Expected Times to Leave a Search Point}
In this subsection, we will prove some bounds on the time to leave a search point that has a Hamming 
distance larger than~$1$ to all improvements.
Let us define by the \emph{epoch} of~$x$ the sequence of iterations where $x$ is the current search point. In contrast to the previously proposed algorithm \sdrlss in \cite{RajabiWittEvo21}, the optimization time of  \sdrlsss for making progress with the current search point~$x$ is also dependent to the progress in the epoch of the second-to-last search point, \ie, the parent of~$x$. In detail, the algorithm in epoch~$x$ starts with parameters~$r_0$ and $B_0$, which are set to the strength escaping the parent of~$x$ and 
the number of fitness calls at that strength divided by $(r_0-1)\ln n$, respectively. Therefore, to analyze the running time, we also need to consider those parameters. Hereinafter, we define $T_{x,r_0,B_0}$ as the number of steps \sdrlsss takes to find an improvement  from the current search point~$x$ with starting radius~$r_0$ (\ie, 
at the beginning of the epoch) and budget~$B_0$ in the current epoch.

We recall the so-called \emph{gap} of the point $x\in \{0,1\}^n $ defined in \cite{RajabiWittGECCO20}
as the minimum Hamming distance to points with the strictly larger fitness function value. Formally, $\gap(x^*)=\infty$ where $x^*$ is an optimum, and for the rest of the points, we define
\[\gap(x)\coloneqq \min\{H(x,y) : f(y)>f(x) , y\in \{0,1\}^n\}. \]
It is not possible to make progress by flipping less than $\gap(x)$ bits of the current search point~$x$, but if the algorithm uses the $s$-flip with $s=\gap(x)$, it can make progress with a positive probability.

In order to estimate the escape time bounds, we consider two cases where $\gap(x)\ge r_0$ and where $\gap(x)< r_0$. In the first case, it costs $B_0$ fitness function calls for the algorithm to increase the radius~$r$ to~$r_0$. Then, the analysis of the rest of the iterations is the same as Theorem~3 for the algorithm~\sdrlss (the algorithm without radius memory)  in \cite{RajabiWittEvo21}. Obviously, the algorithm is asymptotically as efficient as \sdrlss in this case. However, in the second case where $\gap(x)< r_0$, the proposed algorithm can be outperformed by  \sdrlss since if it fails to improve at every radius, the algorithm meets larger strengths, which are costly. It means that although the gap size of the current search point is less than its parent, the escape time is the same in the worst case scenario. However, in the the rest of the paper, we show that this is not harmful to the optimization time because it is captured by the escape time of the previous epoch, and after a short time, the radius ``recovers'' and is set to the gap of the current search point, as proved in Lemma~\ref{lem:realradius}.

Concretely, we present the next theorem presenting escape time bounds and prove it by the end of this subsection. 

\begin{theorem} \label{theo:leavex}
Let $x\in\{0,1\}^n$ be the new current search point of \sdrlsss with starting radius~$r_0$, budget~$B_0$, and $R\ge n^{3+\epsilon}$ for an arbitrary constant $\epsilon>0$
 on a  pseudo-Boolean function $f\colon\{0,1\}^n\to \R$.
	Define $T_{x,r_0,B_0}$ as the time to create a strict improvement if $m=\gap(x)$. Then, 
	\begin{enumerate}
	    \item \label{theo:leavex:greater_m} if $m \ge r_0$,
	    \begin{align*}
    		\expect{T_{x,r_0,B_0}} \le     
    		\begin{cases}
    		\binom{n}{m}\cdot \mathord{O}\mathord{\left(1+\frac{m^3}{n-2m}\ln R\right)} & \text{if } m<n/2,\\
    		O(2^nn \ln R) & \text{otherwise,}
    		\end{cases}
	    \end{align*}
	
	\item \label{theo:leavex:greater_r} or if $m < r_0$, 
	    \begin{align*}
    		\expect{T_{x,r_0,B_0}} \le     
    		\begin{cases}
    		\binom{n}{r_0}\cdot \mathord{O}\mathord{\left(1+\frac{r_0^2}{n-2r_0}\ln R\right)} & \text{if } r_0<n/2,\\
    		O(2^n \ln R) & \text{otherwise.}
    		\end{cases}	    \end{align*}
	\end{enumerate}
\end{theorem}

For proving Theorem~\ref{theo:leavex}, we need some definitions and lemmas as follows. Let \textbf{phase~$r$} consist of all points in time where radius~$r$ is used in the algorithm. 
Let \emph{$E_r$} be the event of \textbf{not} finding the optimum within phase~$r$. In Lemma~\ref{lem:failure-probability}, we show that the failure probability is at most $1/R$ in phases~$r$ to $n/2$ and zero in the last phase (\ie, at radius~$n$). In the following lemma, we show that in each radius which is at least the gap size, the algorithm makes progress with high probability. 

\begin{lemma} \label{lem:failure-probability}
	Let $x\in\{0,1\}^n$ be the current search point of \sdrlsss on a pseudo-Boolean fitness function $f\colon\{0,1\}^n\to \R$ 
	and let $m=\gap(x)$ and $B\ge\binom{n}{m} \ln R$. 
	Then 
	\begin{align*}
		\prob{E_r} \le     
		\begin{cases}
		1 & \text{ if }  r<m\\
		\frac{1}{R} & \text{ if } m \le r<\frac n2\\
		0 & \text{if } r=n.
		\end{cases}
	\end{align*}
\end{lemma}

\begin{proof}
According to the definition of the $\gap$, the algorithm can not make progress where the strength is less than the gap size, so in this case, $\prob{E_r}=1$.
Now, assume $r \ge m$. During phase~$r$ (\ie, at radius~$r$), the algorithm spends $\min\{B,\binom{n}{m} \ln R\}=\binom{n}{m} \ln R$ steps at strength~$m$ until it changes the strength or radius. Then, the 
		probability of not improving at strength~$s=m$ is at most
		\begin{align*}
		\prob{E_r} & = \left(1-\binom{n}{m}^{-1}\right)^{\binom{n}{m}\ln R} \le \frac{1}{R}.
		\end{align*}
		
	During phase~$n$, the algorithm does not change the radius anymore, and it continues to flip $s$ bits with different $s$ containing $m$ until making progress, so the probability of eventually failing to find the improvement in this phase is~$0$.
\end{proof}

Let \emph{$E_{i}^{j}$} for $j>i$ be the event of not finding the optimum 
during phases~$i$ to~$j$. In other words, $E_i^j=E_i\cap \dots \cap E_{j}$. Obviously, we have $\emph{$E_{i}^{i}$}=E_i$ and $E_i^j=\emptyset$ for $j<i$.

The next lemma is used to analyze situations where the algorithm does not find an improvement in the first $k$ phases,  where $k$ is some number depending on the application of the lemma. In the proof of the lemma, we pessimistically do not consider possible improvements at distances larger than the gap size.

\begin{lemma} \label{lem:runtime-of-Em}
    Let $x\in\{0,1\}^n$ with $m=\gap(x)<n/2$ be the current search point of \sdrlsss with $R\ge n^{3+\epsilon}$ for any arbitrary constant $\epsilon>0$
 on a  pseudo-Boolean function $f\colon\{0,1\}^n\to \R$. $T^\prime_{x,r_0,B_0}$ is defined as the number of steps to find an strict improvement. Then, for $k\ge\max\{m-1,r_0\}$, we have
        \begin{align*}
		\expect{T^\prime_{x,r_0,B_0}\mid E_1^{k}} \le 
		\begin{cases}
    		\binom{n}{m}\cdot \mathord{O}\mathord{\left(1+\left(\frac{m^3}{n-2m}\ln R\right)\right)} & \text{if } k=m-1\\
    		\binom{n}{k}\cdot \mathord{O}\mathord{\left(1+\frac{k^2}{n}\ln R\right)}& \text{if } k\ge m
    		\end{cases},
	\end{align*}
	where $E_i^j$ is the event of not finding an optimum during phases~$i$ to $j$ (included).
\end{lemma}

\begin{proof}
According to the definition, we have
\[\expect{T^\prime_{x,r_0,B_0}\mid E_1^k} = \expect{T^\prime_{x,r_0,B_0}\mid E_1 \cup \dots \cup E_{k}}.\]
Since the algorithm cannot make progress in a radius smaller than the gap, $\prob{E_i}=1$ for $i<m$, so the last term equals
\[ \expect{T^\prime_{x,r_0,B_0}\mid E_m \cup \dots \cup E_{k}}, \]

and using the law of total probability, the term equals
\begin{align*}
	\underbrace{ \sum_{i=k+1}^{\floor{n/2-1}} \expect{T^\prime_x \mid E_{k+1}^{i-1}\cup \overline{E_i} } \prob{E_{k+1}^{i-1}\cup \overline{E_i}}}_{=:S_1} + \underbrace{\expect{T^\prime_x \mid E_{k+1}^{\floor{n/2}}\cup \overline{E_n}} \prob{E_{k+1}^{\floor{n/2}}\cup \overline{E_n}}}_{=:S_2}.
	\end{align*}
	
	To interpret the last formula, we recall phase~$r$ as all points of time where radius~$r$ is used in the algorithm. 
	In each term in $S_1$, Variable~$i$ represents the phase in which the algorithm makes progress (\ie, $\overline{E_i}$) and not in smaller phases, \ie, phases from~$k+1$ to $i-1$ (\ie, $E_{k+1}^{i-1}$). Thus, all cases where making improvement happens in one of the phases ranging from~$k+1$ to $\floor{n/2-1}$ are considered in $S_1$, and the last case of phase~$n/2$ is computed in $S_2$. In this manner, we consider all possible cases of success.
	
  	In order to estimate $\prob{E_{k+1}^{i-1}\cup \overline{E_i}}$, because of the assumption $k\ge\max\{m-1,r_0\}$, the budget $B_0$ is not effective in the threshold value so we can use Lemma~\ref{lem:failure-probability} resulting in
  	\[\prob{E_{k+1}^{i-1}\cup \overline{E_i}}\le \prob{E_{k+1}^{i-1}} = \prod_{j=k+1}^{i-1}\prob{E_j}<R^{-(i-k-1)}.\]
  	Note that we consider only improvements in the Hamming distance of the gap size. Now, For $S_1$, we compute
	\begin{align*}
	&\sum_{i=k+1}^{\floor{n/2-1} } \expect{T^\prime_x \mid E_{k+1}^{i-1}\cup \overline{E_i}} \prob{E_{k+1}^{i-1}\cup \overline{E_i}}   \\
	& \qquad \le \sum_{i=k+1}^{\floor{n/2-1}} \left( \sum_{r=1}^{i-1}\sum_{s=1}^{r} \binom{n}{s}\ln R + \sum_{s=1}^{m-1} \binom{n}{s}\ln R + \binom{n}{m} \right) \cdot R^{-(i-k-1)}. \end{align*}
	This is because in the last phase, the success can happen with strength~$m$, so we do not consider the strengths larger than $m$ in the last phase. Also, in the last phase, the algorithm makes progress in $\binom{n}{m}$ iterations in expectation.
	
	Now, since $m\le i$ resulting in $\sum_{s=1}^{m-1} \binom{n}{s} \le \sum_{r=1}^{i-1}\sum_{s=1}^{r} \binom{n}{s}$ so the last term is bounded from above by
	\begin{align*}
	& \qquad   \sum_{i=k+1}^{\floor{n/2-1}} \left( \sum_{r=1}^{i-1}2\sum_{s=1}^{r} \binom{n}{s}\ln R + \binom{n}{m} \right) \cdot R^{-(i-k-1)} \\
	 &\qquad \le 2\sum_{i=k+1}^{\floor{n/2-1}} \sum_{r=1}^{i-1} \left( \frac{n-r+1}{n-2r+1}\binom{n}{r}\ln R + \binom{n}{m}\right)   R^{-(i-k-1)} \\
	&\qquad \le 2 \sum_{i=k+1}^{\floor{n/2-1}} \left( \left(\frac{n-i}{n-2i+3}\right)^2\binom{n}{i-1}\ln R + \binom{n}{m} \right)   R^{-(i-k-1)}\\
	&\qquad \le 2 \sum_{i=k+1}^{\floor{n/2-1}} \left( \left(1+\frac{i^2}{n-2i}\right)\binom{n}{i-1}\ln R + \binom{n}{m} \right)   R^{-(i-k-1)}
	\end{align*}
	In the second and third inequalities, we applied the first inequality in Lemma~\ref{lem:partial-sum} to eliminate two summations. 
	
	Now, via the second inequality in Lemma~\ref{lem:partial-sum}, and then excluding the first term from the summation, we bound the last term from above by
	\begin{align*}
	& 2\left(1+\frac{(k+1)^2}{n-2k-2}\right) \frac{k+1}{n-k-1} \binom{n}{k+1} \ln R  + \binom{n}{m} \\ 
	&\qquad  + 4\sum_{i=k+2}^{\floor{n/2-1}}  \left(1+\frac{i^2}{n-2i}\right)\left(\frac{n-k-1}{k+1}\right)^{i-k-2}\binom{n}{k+1} \ln R  \cdot R^{-(i-k-1)} \\
	& = 2\left(1+\frac{(k+1)^2}{n-2k-2}\right) \frac{k+1}{n-k-1} \binom{n}{k+1} \ln R + \binom{n}{m} \\
	& \qquad  + \frac{4\ln R}{R}\binom{n}{k+1} \sum_{i=k+2}^{\floor{n/2-1}}  \left(1+\frac{i^2}{n-2i}\right)\left(\frac{n}{Rk}\right)^{i-k-2} .
	\end{align*}
	Using the fact that $R\ge n^{3+\epsilon}$ and $i>k$, the 
	last expression 
	is bounded from above by $O(\frac{ \ln R}{R}n^2\binom{n}{k+1})=O(\frac{ \ln R}{R}\frac{n^3}{k}\binom{n}{k})=o(\binom{n}{k})$. If $k \ge m$, the last inequality is bounded by
	\begin{align*}
	     2\left(1+\frac{(k+1)^2}{n-2k-2}\right) \binom{n}{k} \ln R +  o\left(\binom{n}{k}  \right) 
	     \le \binom{n}{k} \cdot \mathord{O}\mathord{\left(  1+\frac{k^2}{n} \ln R \right)}.
	\end{align*}
	However, if $k=m-1$, the inequality is bounded by
	\begin{align*}
	    &2\left(1+\frac{m^2}{n-2m}\right) \frac{m}{n-m+1} \binom{n}{m} \ln R + \binom{n}{m} +
	 o\left(\binom{n}{m}  \right) \\ & \qquad \le \binom{n}{m}O\left(1+\frac{m^3}{n-2m}\ln R\right).
	\end{align*}
	
Regarding $S_2$, when radius~$r$ is increased to $n$, the algorithm mutates $s$ bits of the the current search point for all possible strengths of~$1$ to $n$ periodically. In each cycle through different strengths, according to lemma \ref{lem:failure-probability}, the algorithm escapes from the local optimum with probability $1-1/R$ so there are  $R/(R-1)$ cycles in expectation via geometric distribution. Besides, each cycle of radius~$n$ costs $\sum_{s=1}^{n}\binom{n}{s}\ln R$. Overall, we have $\frac{R}{R-1}\sum_{s=1}^{n}\binom{n}{s}\ln R$ extra fitness function calls if the algorithm fails to find the optimum in the first $\floor{n/2-1}$ phases happened with the probability of $R^{-(\floor{n/2}-k-1)}$. Thus, we have

\begin{align*}
    &\expect{T_{x,r_0,B_0}^\prime \mid E_{k+1}^{\floor{n/2}}\cup \overline{E_n} }\prob{E_{k+1}^{\floor{n/2}}\cup \overline{E_n}}  \\ &\qquad \le \left( \sum_{r=1}^{\floor{n/2}}\sum_{s=1}^{r}\binom{n}{s}\ln R + \frac{R}{R-1} \sum_{s=1}^{n}\binom{n}{s}\ln R \right) R^{-(\floor{n/2}-k-1)} \\ 
    &\qquad \le \left( 3\sum_{r=1}^{\floor{n/2}}\sum_{s=1}^{r}\binom{n}{s}\ln R \right) R^{-(\floor{n/2}-k-1)} \\
    &\qquad \le \left( 3 n^2\binom{n}{\floor{n/2}}\ln R \right) R^{-(\floor{n/2}-k-1)} \\
    &\qquad \le \left( 3 n^2  \left( \frac{n-k}{k}\right)^{\floor{n/2}-k} \binom{n}{k}\ln R \right) R^{-(\floor{n/2}-k-1)} \\
    &\qquad \le \left( 3 n^2  \left( \frac{n}{Rk}\right)^{\floor{n/2}-k} \binom{n}{k}\ln R \right) \le o\left(\binom{n}{k}\right)
\end{align*}

Altogether, we finally have $\expect{T^\prime_{x,r_0,B_0}\mid E_1^k}=S_1+S_2$, resulting in the statement.
\end{proof}

Now, by using Lemma~\ref{lem:runtime-of-Em}, we prove Theorem~\ref{theo:leavex}.
\begin{proofof}{Theorem~\ref{theo:leavex}}
	Regarding Part~\ref{theo:leavex:greater_m}, the starting radius~$r_0$ is less than or equal to the gap of the current search point, \ie, $r_0\le m=\gap(x)$. For $m<n/2$, using Lemma~\ref{lem:runtime-of-Em} with $k=m-1$, we have
        $\expect{T_{x,r_0,B_0}} \le     \binom{n}{m}\cdot \mathord{O}\mathord{\left(1+\frac{m^3}{n-2m}\ln R\right)}$.
	
	For $m\ge n/2$, the algorithm is not able to make an improvement for radius~$r$ less than $n/2$. It can be pessimistically assumed that $B_0=\infty$. However, as radius~$r$ is increased to~$n$, the algorithm mutates~$m$ bits of the the current search point for all possible strengths of~$1$ to $n$ periodically. Thus, according to Lemma~\ref{lem:failure-probability}, the algorithm escapes from the local optimum with probability at least $1-1/R$ so there are at most $R/(R-1)$ cycles in expectation  in this phase (\ie, at radius~$n$) by the geometric distribution. Finally, we compute
	\begin{align*}
	\expect{T_{x,r_0,B_0} \mid E_1^{m-1}}  &< \sum_{ i=1}^{\floor{n/2-1}} \sum_{ j=1}^{i}\binom{n}{j} \ln R 
	+ \frac{R}{R-1} \sum_{ i=1}^{n} \binom{n}{i} \ln R \\
	 & \le \mathord{O}\mathord{\left(2^nn\ln R\right)}.
	\end{align*}
	
	Regarding Part~\ref{theo:leavex:greater_r}, the algorithm starts with a radius which is larger than the gap size. Thus, in phase~$r_0$, if $B_0\ge\binom{n}{m}\ln R$, then the probability of failure in this phase would be $\prob{E_{r_0}}\le 1/R$ via Lemma~\ref{lem:failure-probability}. Otherwise, the probability of failure can be greater than~$1/R$.
	
	We pessimistically assume that $\prob{E_{r_0}}=1$, which means that the algorithm does not make progress in this phase.
	Therefore, using Lemma~\ref{lem:runtime-of-Em} with $k=r_0$ the time for leaving the point is
	$
	    \binom{n}{r_0}\cdot \mathord{O}\mathord{\left(1+k^2/n\ln R\right)}
	$.
	
		For $r_0\ge n/2$, we pessimistically assume that the algorithm is not able to make an improvement for the strengths less than $r_0$, costing $(1/\ln n)\cdot (r_0-1)\cdot B_0$ iterations. Then the algorithm tries strength~$r_0$. If no improvement is found, 
		the radius is increased to~$n$, and the algorithm mutates $m$ bits of the the current search point for all possible strengths~$m$ from~$1$ to $n$ periodically. Thus, like part (a), according to Lemma~\ref{lem:failure-probability}, the algorithm escapes from the local optimum with probability at least $1-1/R$, so there are at most  $R/(R-1)$ cycles in expectation through the geometric distribution in this phase. Finally, we compute
	\begin{align*}
	\expect{T_{x,r_0,B_0} \mid E_1^{m-1}}  &< \frac{(r_0-1) B_0}{\ln n} + \binom{n}{r_0} \ln R  + \frac{R}{R-1} \sum_{ i=1}^{n} \binom{n}{i} \ln R \\
	 & = \mathord{O}\mathord{\left(2^n\ln R\right)}. 
	\end{align*}
\end{proofof}

\subsection{Expected Optimization Times}
In this subsection, we will prove a crucial technical 
lemma on recover times and use it to 
obtain bounds on the expected 
optimization time on unimodal and jump functions.

\paragraph{Recover times for strengths.}
In the previous subsection, we analyzed the time of \sdrlsss for leaving only a single search point. We observed that the duration of epochs depended 
on the starting radius denoted as $r_0$ set from the previous epoch. This can be inconvenient to estimate an upper bound on the running time on an arbitrary function. Therefore, in the following lemma,  we show that if the algorithm uses larger strengths than the gap size to make progress, after a relatively small number of iterations the algorithm chooses the gap size of a current search point as the strength.

\begin{lemma}
\label{lem:realradius}
    Let $x\in\{0,1\}^n$ with $m=\gap(x)<n/2$ be the current search point of \sdrlsss with $R\ge n^{3+\epsilon}$ for an arbitrary constant $\epsilon>0$
 on a  pseudo-Boolean function $f\colon\{0,1\}^n\to \R$. Assume that the radius is~$k>m$. Define $S_{x}$ as the number of iterations spent from that point in time on until the algorithm sets the radius to at most the gap size of the current search point. Assume that $B_0$ is the value of the variable~$B$ in the beginning. Then, for $B_0\ge\binom{n}{m}\ln R$, we have
		\[ \expect{S_{x}} \le \mathord{o}\mathord{\left( \binom{n}{k-1}\right)},\]
		and for $B_0<\binom{n}{m}\ln n$, we have
		\[ \expect{S_{x}} \le \mathord{o}\mathord{\left( R\binom{n}{k-1}\right)}.\]
\end{lemma}

The idea of the proof is that in the case of making progress with larger strengths than~$m$ or failing to improve 
and increasing strength and radius even further, the algorithm also tries all smaller strengths often enough, 
more precisely, as often as the threshold value $\binom{n}{\gap(x)}\ln R$ 
that would hold if the current search point was~$x$,  $s=\gap(x)$ and $B=\infty$ in a phase of Algorithm~\ref{alg:sdrls_star}. Thus, the algorithm can make progress when the strength 
equals the current gap with good probability.

\begin{proofof}{Lemma~\ref{lem:realradius}}
 We recall the epoch of $x$ as the sequence of iterations where $x$ is the current search point.
We assume that the gap size of the current search point does not become equal to or larger than the current radius value in all epochs. Otherwise, $S_x$ is no bigger than the following estimation.

Assume $x^\prime$ is the current search point and $r^\prime$ is a radius value larger than~$g_{x^\prime}\coloneqq\gap(x^\prime)$. Note that $g_{x^\prime}$ equals $m$ in the beginning (in the first epoch), but it may be different when a strict improvement is made.

We now claim that for each at most~$\ln n \cdot r^\prime \binom{n}{r^\prime}\ln R$ iterations with strength~$r^\prime$ in phase~$r^\prime$  (\ie, at radius~$r^\prime$) (even in different epochs), the algorithm uses smaller strengths including strength~$g_{x^\prime}$ for $\binom{n}{g_{x^\prime}}\ln R$ iterations. After proving the claim, we can show that with probability~$1-1/R$ the algorithm makes progress with the strength which equals the gap size of the current search point via Lemma~\ref{lem:failure-probability}.

To prove this claim, we consider two cases. First, if the algorithm does not make progress with strength~$r^\prime$, then in the next phase the algorithm uses strengths smaller than~$r^\prime$ including $g_{x^\prime}$ in $\binom{n}{g_{x^\prime}}\ln R$ iterations. Thus, $\binom{n}{r^\prime}\ln R$ iterations with strength~$r^\prime$ are enough for satisfying the claim in this case.

In the second case, assume that the algorithm makes an improvement with strength~$r^\prime$ in its $u^{\mathit{th}}$ attempt. For the next epoch, the algorithm tries strength~$g_{x^\prime}$ for $(1/\ln n)\cdot u/(r^\prime-1)$ times (according to the variable~$B$ in Algorithm~\ref{alg:sdrls_star}). Assume that $u_1,\dots,u_{\ell}$ are the counter values where the algorithm makes progress with strength~$r^\prime$. Thus, after at most $\sum_{i=1}^\ell u_i \le (\ln n) \cdot  (r^\prime-1)\binom{n}{g_{x^\prime}}\ln R-1$ improvements with strength~$r^\prime$, the number of iterations with strength~$g_{x^\prime}$ is at least $\sum_{i=1}^\ell (1/\ln n)\cdot u_i/(r^\prime-1) \ge \binom{n}{g_{x^\prime}}\ln R-1$. In the next epoch, there would be at most $\binom{n}{r^\prime}\ln R$ iterations with strength~$r^\prime$ for having the last required iteration. 

Overall, it costs less than $(\ln n)\cdot r^\prime\binom{n}{g_{x^\prime}}\ln R+\binom{n}{r^\prime}\ln R<(\ln n)\cdot r^\prime\binom{n}{r^\prime}$ iterations to observe $\binom{n}{g_{x^\prime}}\ln R$ iterations with strength~$g_{x^\prime}$, and consecutively with a probability of at least~$1-1/R$, the algorithm makes progress with strength~$g_{x^\prime}$ via Lemma~\ref{lem:failure-probability}.

If we pessimistically assume that after each $\binom{n}{g_{x^\prime}}\ln R$ iterations with strength~$g_{x^\prime}$, the radius is increased by one, by the law of total probability and Lemma~\ref{lem:partial-sum}, the expected 
number of steps at larger 
strengths than the gap of the current search point is at most
\begin{align*}
&\sum_{r=k}^{\floor{n/2-1}} \frac{n-(r-1)}{n-(2r-1)}r\ln n\binom{n}{r} \ln R \cdot R^{-(r-k-1)} \\ 
& \quad <R \cdot \sum_{r=k}^{\floor{n/2-1}} n^2\ln n \cdot \left(\frac{n-k+1}{k-1}\right)^{r-k+1}\binom{n}{k-1}\ln R\cdot R^{-(r-k)}
\end{align*}
for the phases ranging from~$k$ to $\floor{n/2-1}$ and at most

\begin{align*}
 &\frac{R}{R-1}\cdot  2n\ln n\binom{n}{\floor{n/2}} \ln R \cdot R^{-(r-m-1)}  \\
 & \quad < R \cdot  4n\ln n\left(\frac{n-k+1}{k-1}\right)^{\floor{n/2}-k+1}\binom{n}{k-1} \ln R \cdot R^{-(\floor{n/2}-k-2)}
\end{align*}
for the last phase  (\ie, at radius~$n$).
Since $R>n^{3+\epsilon}$, both are bounded from above by $\mathord{o}\mathord{\left(R\binom{n}{k-1}\right)}$.

However, in phase~$k$, before reaching the strength~$k$, the algorithm uses strength~$m$  $B_0$ times. If $B_0\ge \binom{n}{m}\ln R$, then with probability at least~$1-1/R$, the algorithm finds an improvement at the Hamming distance corresponding to the gap size, resulting in \[\expect{S_x}\le \frac{1}{R} \cdot  \mathord{o}\mathord{\left(R\binom{n}{k-1}\right)}=\mathord{o}\mathord{\left(\binom{n}{k-1}\right)}. \]
\end{proofof}

\paragraph{Analysis on unimodal functions}
On unimodal functions, the gap of all points in the search space (except for global optima) is one, so the algorithm can make progress with strength~1.
In the following theorem, we show how \sdrlsss behaves on unimodal functions compared \rls using an upper bounds based on the fitness-level method \citep{WegenerMethods}. The proof is similar to the proof of Lemma~4 in \cite{RajabiWittEvo21}.

\begin{theorem} \label{lem:unimodalrlsstar}
	Let $f\colon\{0,1\}^n\to \R$ be a unimodal function and consider \sdrlsss with $R\ge \max\{S, n^{3+\epsilon}\}$ for an arbitrary constant $\epsilon>0$ where $S$ is an upper bound on the number of strict improvements during the run, 
	\eg, $S = \card{\im f}$.	Then, with probability at least~$1-S/R^2$, \sdrlsss never uses strengths larger than~$1$ and behaves stochastically like \rls before finding an optimum of~$f$. 

	Denote by $T$ the runtime of \sdrlsss on~$f$. Let $f_i$ be the $i$-th fitness value of an increasing order of all fitness values in $f$ and $s_i$ be a lower bound for the probabilities that \rls finds an improvement from search points with fitness value $f_i$, then 
	$
	\expect{T} \le  \sum_{i=1}^{ \card{\im f} } 1/s_i + \mathord{o}\mathord{\left(n\right)}
	$.
\end{theorem}

\begin{proof}
	The algorithm~\sdrlsss uses strength~$1$ for $\binom{n}{m}\ln R$ times when 
	the radius is~1 and $\binom{n}{m}\ln R$ times when the radius is~2 but the strength is still~1. (Only considering the first 
	case would not be sufficent for the result of this lemma.)  Overall, the algorithm tries $2\binom{n}{m}\ln R$ steps with strength~1 before setting the strength to~$2$.
	
	As on unimodal functions, the gap of all points is~$1$, the probability of not finding and improvement is $$\left(1-\binom{n}{m}^{-1}\right)^{2\binom{n}{m}\ln R} \le \frac 1{R^2}.$$
	This argumentation holds for each improvement that has to be found. 
	Since at most $\card{\im f}$ improving steps happen before 
	finding the optimum, by a union bound the probability of \sdrlsss ever increasing the strength beyond~$1$ is at most 
	$S\frac{1}{R^2}$, which proves the lemma.
	
	To prove the second claim, we consider all fitness levels $A_1, \dots, A_{\card{\im f}}$ such that $A_i$ contains search points with fitness value $f_i$ and sum up upper bounds on the expected times to leave each of these fitness levels. Under the 
	condition that the strength is not increased before leaving a fitness level, the worst-case time to leave fitness level~$A_i$ is $1/s_i$ similarly to \rls. Hence,
	we bound the expected optimization time of \sdrlsss from above 
	by adding the waiting times on all fitness levels for \rls, which is given by $\sum_{i=1}^{ \card{\im f} } 1/s_i$. 
	
	We let the random set~$W$ contain the search points from which \sdrlsss does not find an improvement within phase~$1$ (\ie, while $r=1$) so the radius is increased. Assume $T_x$ is the number of iterations spent where the radius is larger than 1 and increasing the radius happening where $x$ is the current search point; formally,  
	\[
	\expect{T} \le \sum_{i=1}^{ \card{\im f} }\frac 1{s_i} + \sum_{x\in W}\expect{T_x}.
	\]

	Each search point selected by the algorithm contributes with probability $\prob{E_1}$ to~$W$. Hence, as S is an upper bound on the number of improvements,
	$\expect{\card{W}}\le S\cdot \prob{E_1}$.
	As on unimodal functions, the gap of all points is~1, by 
	Lemma \ref{lem:realradius}, we compute
	\begin{align*}
	\sum_{x\in W}\expect{T_x} 
	&\le S\cdot \prob{E_1}\cdot \expect{T_x\mid \gap(x)=1} \\ & \le S\cdot R^{-1}o\left(\binom{n}{1}\right)=o\left(n\right).
	\end{align*}
	
	Thus, we finally have
	\begin{align*}
	\expect{T} \le \sum_{i=1}^{ \card{\im f} }\frac 1{s_i}+o\left(n\right),
	\end{align*}
	as suggested.
\end{proof}

\paragraph{Analysis on \jumpo}
We use the results developed so far to prove a bound on a newly designed function called \jumpo with two parameters $m$ and $c$. Formally,
\[
\jumpo_{m,c} = 
\begin{cases}
m + \ones{x} & \text{ if $\ones{x}\le n-m-c$ or $\ones{x}\ge n-c$,}\\
n-\ones{x}-c & \text{ otherwise.}
\end{cases}
\]

This function can be considered as a variant of well-known jump benchmark. In this function, we move the location of the jump with size~$m$ to an earlier point than it is in the function \jump such that after the jump there is a unimodal sub-problem behaving like \onemax of length $c$. The \jump function is a special case of \jumpo with $c=0$, \ie
$\jumpo_{m,0}=\jump_m$.

The following theorem shows that \sdrlsss optimizes \jumpo in a time that is essentially determined 
by the time to overcome the gap only. The proof idea is that the algorithm can quickly re-adapt 
its radius value to the gap size of the current search point after escaping the local optimum. 

\begin{theorem}
\label{theo:jumpoff}
	Let $n\in \N$. For all $2\le m<O(\ln n)$ and $0\le c<O(\ln n)$, the expected runtime $\expect{T}$ of \sdrlsss with $R\ge n^{3+\epsilon}$ for an arbitrary constant $\epsilon>0$ on $\jumpo_{m,c}$ satisfies 
	$\expect{T} = \mathord{O}\mathord{\left(\binom{n}{m}\right)}$,
	conditioned on an event 
    that happens 
	with probability~$1-o(1)$.
\end{theorem}

\begin{proof}
	Before reaching the plateau consisting of all points of $n-m-c$ 1-bits, $\jumpo_{m,c}$ is equivalent to \onemax; hence, 
	according to Lemma~\ref{lem:unimodalrlsstar}, the expected time \sdrlsss takes to reach the plateau is at most $O(n\ln n)$. Note that this bound was obtained 
	via the fitness level method with $s_i=(n-i)/n$ as 
	minimum probability for leaving the set of search points 
	with $i$ one-bits.
	
	Every plateau point $x$ with $n-m-c$ one-bits satisfies $\gap(x)=m$ according to the definition of $\jumpo_{m,c}$. 
	Thus, 	using Theorem~\ref{theo:leavex}, the algorithm finds one of the $\binom{m+c}{m}$ improvements within expected time at most
	 $\binom{n}{m}$.
	
	According to Lemma~\ref{lem:failure-probability}, this success happens with strength~$m$ (and not larger) with probability at least~$1-1/R$.
	
	Now, we compute the probability that $B_0<n \ln R$, resulting from
	making progress within less steps than $\ln n (m-1)\binom{n}{1}\ln R$ in the previous epoch.
	\begin{align*}
	    &\prob{u<(m-1)n\ln n \ln R} \le 1-\left(1-\frac{\binom{m+c}{m}}{\binom{n}{m}}\right)^{(m-1)n\ln n \ln R} \\
	    &\le (m-1)n\ln n\ln R\frac{\binom{m+c}{m}}{\binom{n}{m}} 
	    \le (m-1)n\ln n\ln R\frac{(e(m+c))^m}{n^m}
	\end{align*}
	According to the assumption on $m$ and $c$, the last term is bounded from above by 
	\[
	(m-1)n\ln n\ln R\frac{(e\ln^2 n)^m}{n^m} \le \frac{m(e\ln n)^{2m+2}}{n^{m-1}}
	\]
	
	This means that with probability at least~$1-o(1)$, the variable~$B$ is not effective in the beginning of the next epoch with strength~$1$.
	 
	 After making progress over the jump, the starting radius~$r_0$ is at least~$m$, although an improvement can be found at the Hamming distance of~$1$, \ie, $\gap(x)=1$. Via Lemma~\ref{lem:realradius} with $k=m$ and $B_0\ge \binom{n}{1}\ln R$, within $o(\binom{n}{m-1})$ steps in expectation, the algorithm sets the radius to~$1$, if it pessimistically does not find the optimum point. Now, the algorithm needs to optimize a \onemax of length~$c$, so its expected time again can be obtained from Lemma~\ref{lem:unimodalrlsstar}, which is at most $O(n\ln n)$.
	 
	 Altogether, $\expect{T} \le O(n\ln n)+ \binom{n}{m}+\binom{n}{m-1}+O(n\ln n)$, conditioned on the mentioned events of having enough iterations with the strength passing the jump part and escaping from the local optimum with strength~$m$, happening
	with probability~$1-o(1)$.
\end{proof}

\section{Speed-ups By Using Radius Memory}
\label{sec:speedups}
In this section, we consider the problem of minimizing a linear
function under a uniform constraint as analyzed
in \cite{NPWGECCO19}: 
given a linear pseudo-Boolean function
$f(x_1,\dots,x_n)=\sum_{i=1}^n w_i x_i$, the aim is to find 
a search point $x$ minimizing
$f$ under the constraint 
$\ones{x}\ge B$ for 
some $B\in\{1,\dots,n\}$. W.\,l.\,o.\,g., 
$w_1\le \dots\le w_n$. 

\citet{NPWGECCO19} obtain a tight worst-case runtime bound $\Theta(n^2)$ for \rlsonetwo and 
a bound for the \oneoneea which is $O(n^2\log B)$ 
and therefore 
tight up to logarithmic factors. We will see in 
Theorem~\ref{theo:sdrlsss-linear-uniform} that with high probability,
\sdrlsss achieves the same bound $O(n^2)$ despite 
being able to search  globally like the \oneoneea. 
Afterwards, we will identify a 
scenario where \sdrlss is 
by a factor of $\Omega(n)$ slower.

We start with the general 
result on the worst-case 
expected optimization time, assuming the set-up of \cite{NPWGECCO19}.

\begin{theorem}
\label{theo:sdrlsss-linear-uniform}
Starting with an arbitrary
initial solution, the expected optimization time of
\sdrlsss with $R\ge n^{3+\epsilon}$, where $\epsilon>0$ 
is an arbitrary constant, on a linear function with a
uniform constraint is $O(n^2)$ conditioned on an event 
that happens with probability
$1-O(1/n)$.
\end{theorem}

\begin{proof}
We follow closely the proof of Theorem~4.2 
in~\cite{NPWGECCO19} who analyze \rlsonetwo. 
The first phase of optimization (covered in Lemma~4.1 of 
the paper) deals with the time to reach a feasible 
search point and proves this to be  $O(n\log n)$ in 
expectation. Since the proof uses multiplicative drift,
it is easily seen that the time is $O(n\log n)$ with 
probability at least $1-O(1/n)$ thanks to the tail bounds
for multiplicative drift 
\citep{LenglerDriftBookChapter}. The second phase 
deals with the time to reach a tight search point (\ie,
containing $B$ one-bits, 
which is $O(n\log n)$ with probability at least $1-O(1/n)$
by the very same type of arguments. The analyses so 
far rely exclusively on one-bit flips so that 
the bounds also hold for \sdrlsss thanks to Lemma~\ref{lem:unimodalrlsstar}, up to a failure event 
of $O(S/R^2)=o(1/n)$ since it holds for the number 
of improvements $S$ that $S\le n$. By definition 
of the fitness function, only tight search points will
be accepted in the following.

The third phase in the analysis from 
\cite{NPWGECCO19} considers the potential 
function $\phi(x)=\lvert\{x_i=1: i>B\}\rvert$ 
denoting the number of one-bits outside the $B$ least 
significant positions. 
At the same time, $\phi(x)$ describes the number 
of zero-bits at the $B$ optimal positions. Given 
$\phi(x)=i>0$ for the current search point, the
probability of improving the potential is 
$\Theta(i^2/n^2)$ since there are $i^2$ improving 
two-bit flips ($i$ choices for a one-bit to 
be flipped to~$0$ at the non-optimal positions 
and $i$ choices for a zero-bit to 
be flipped to~$1$ at the other positions). 
This results in an expected 
optimization 
time of at most $\sum_{i=1}^\infty O(n^2/i^2) = O(n^2)$ 
for \rlsonetwo.

\sdrlsss can achieve the same time bound since after every
two-bit flip, the radius memory only allocates 
the time for the last improvement (via two-bit flips) 
to iterations trying strength~$2$. Hence, as long as 
no strengths larger than~$2$ are chosen, the 
expected optimization time of \sdrlsss is $O(n^2)$.

To estimate the failure probability, we need a bound on the number 
of strict improvements of~$f$, which may be larger than the number of improvements 
of the $\phi$-value since steps that flip a zero-bit and a one-bit both 
located in the prefix or suffix may be strictly improving without changing~$\phi$.
Let us assign a value to each zero-bit, representing the number of one-bits to its left (at higher indices). In other words, each of these values shows the number of one-bits that can be flipped with the respective zero-bit to make a strict improvement. 
Let us define by $S$ the sum of these values. Clearly, 
$S\le n(n-1)/2 = O(n^2)$. 
Now, we claim that each strict improvement decreases $S$ by
at least one, resulting in bounding the number of strict improvements from above by $O(n^2)$. Assume that in a strict improvement, the algorithm flips one one-bit at  position~$i$ and one zero-bit at position~$j$. Obviously, $j<i$. 
The corresponding value for the zero-bit at the new position~$i$ is less than the corresponding number for the zero-bit at position~$j$ before flipping because there is at least one one-bit less for the new zero-bit, 
which was the one-bit at the position of $i$. 
Altogether, 
the number of strict improvements at strength~$2$ is at most $O(n^2)$. 

The proof is completed by noting that 
the strength never exceeds~$2$ with probability 
at least $1-O(n^2)/R=1-O(1/n)$ thanks to 
Lemma~\ref{lem:failure-probability} and a union bound. 
\end{proof}

With a little more effort, including 
an application of Lemma~\ref{lem:realradius}, 
the bound from
Theorem~\ref{theo:sdrlsss-linear-uniform} can be turned 
into a bound on the expected optimization time. We 
will see an example of such arguments later in 
the proof of Theorem~\ref{theo:linearfunctions-approx} and in the 
proof of Theorem~\ref{theo:mst}.

We now illustrate why the original \sdrlss is less 
efficient on linear functions under uniform constraints
than \sdrlsss. 
To this end, we study 
the following instance: the weights of the 
objective function  are 
$n$ pairwise different natural numbers (sorted increasingly),  
and the constraint bound is $B=n/2$, \ie, 
only search points having at least~$n/2$ 
one\nobreakdash-bits are valid. Writing 
search points in big-endian as 
$x=(x_n,\dots,x_1)$, we assume 
the point $1^{n/2}0^{n/2}$ as starting point 
of our search heuristic. The optimum  is then $0^{n/2}1^{n/2}$ since the $n/2$
one\nobreakdash-bits are at the least significant 
positions. We call the latter positions the 
suffix and the other the prefix.

Considering the potential $\phi(x)$ 
defined as
the number of one-bits in the prefix, we note 
that the expected time for \rlsonetwo and for  
\sdrlsss (up to a failure event) to reduce the 
potential from its initial value~$n/4$ to $n/8$ 
is $O(n)$ since during this period there is 
always an improvement probability of at least 
$\Omega((n/8)^2/n^2)=\Omega(1)$. We claim  that \sdrlss needs time $\Omega(n^2\log n)$ 
    for this since after each improvement 
    the strength and radius are reset to~$1$, resulting in $\Omega(n)$ phases where 
    the algorithm is forced to iterate 
    unsuccessfully with 
    strength~$1$ until the threshold $\binom{n}{1}\ln R$ is reached.  By contrast, 
    \sdrlsss will spend $O(n\ln n)$ steps 
    to set the strength and radius to~$2$. Afterwards, during the $n/4$ improvements 
    it will only spend an 
    expected number of $O(1)$ steps at strength~$1$ 
    before it returns to strength~$2$ and finds 
    an improvement in expected time $O(1)$. The 
    total time is $O(n\log n+n)=O(n\log n)$.
    Based on these ideas, we formulate the following theorem.

\begin{theorem} \label{theo:linearfunctions-approx}
Consider a linear function with $n$ pairwise different weights under 
uniform constraint with $B=n/2$ and  let $S_b=\{\phi(y)\le b\mid y\in \{0,1\}^n\}$. Starting with a feasible solution~$x$ such that $x$ contains $B$ one-bits and $\phi(x)=a$, the expected time to find a search point in $S_b$ for \sdrlsss with $R\ge n^{3+\epsilon}$ for an arbitrary constant $\epsilon>0$ is at most \[ \mathord{O}\mathord{\left(n\ln R+n^2\sum_{i=b}^{a}\frac 1{i^2}\right)}.\] 

For \sdrlss with $R\ge n^{3+\epsilon}$ for an arbitrary constant $\epsilon>0$ it is at least \[\Omega\left(P\cdot n\ln R+ n^2\sum_{i=b}^{a}\frac{1}{i^2}\right),\]
where $P$ is the number of improvements with strengths larger than~1 and $P>a-b$ with probability~$1-o(1/n)$.
\end{theorem}
\begin{proof}
We first find an upper bound for \sdrlsss. First, \sdrlsss spends $O(n\ln R)$ steps to set the strength and radius to~$2$. Afterwards, the number of iterations with strength~1 is $1/\ln n$ times of the number of iterations with strength~2, conditioned on not exceeding the threshold when $r=2$, which will be studied later.
    
    When the strength equals 2, the probability of improving the potential is $\Theta(i^2/n^2)$, resulting in the expected time of $\Theta(n^2/i^2)$ for each improvement. Thus, in expectation, there are $n^2\sum_{i=b}^a1/i^2$ iterations to find a search point in $S_b$. In the case that the counter exceeds the threshold, happening with probability~$1/R$ for each improvement using Lemma~\ref{lem:failure-probability},
    it costs $o(\binom{n}{2})$ iterations with different strengths in expectation, according to Lemma~\ref{lem:realradius}, to set the radius to~2 again.
     Since the number of fitness improvements at strength~$2$ is at most~$n^2$ (see the proof of Th.~\ref{theo:linearfunctions-approx}), we obtain a failure probability of at most 
     $n^2\cdot 1/R = o(1/n)$  for the event 
    of exceeding the threshold. Hence, the expected number iterations with different strengths 
    is $o((1/n)\binom{n}{2})=o(n)$ and therefore  a lower-order term of the claimed bound on the 
    expected time to reach~$S_b$.
    
    In case of a failure, we 
    repeat the argumentation. Hence, after an expected 
    number of $1/(1-o(1/n))=1+o(1)$ repetitions no failure 
    occurs and $S_b$ is reached. 
    
    Overall, the expected time for \sdrlsss to 
    find a search point in $S_b$ is at most
\[(1+o(1))\left(n\ln R+(1+1/\ln n)\cdot \sum_{i=b}^a\frac{n^2}{i^2}\right)
= \mathord{O}\mathord{\left(n\ln R+n^2\sum_{i=b}^{a}\frac 1{i^2}\right)}.\]

    In order to compute a lower bound on the optimization time of \sdrlss, we claim that the number of potential improvements is at least~$a-b$ with probability~$1-o(1/n)$, \ie, each improvement decreases the potential function roughly by at most~one in expectation. 
    
    As long as the strength does not become greater than~$2$, the number of one-bits remains $B$, so when the strength is at most~2, the algorithm can only improve the potential by~1 since it cannot make progress by flipping at least~two zero-bits in $B$ least significant positions.
    The radius becomes~3 with probability at most~$1/R$ for each improvement at strength~2. Since there are at most $n^2$ improvements, the probability of not increasing the current radius to~$3$ during the run is at least~$1-n^2/R=1-o(1/n)$.
    
    Now, since the algorithm spends $n\ln R$ steps for each improvement, the number of steps with strength~1 is $\Omega((a-b)n\ln R)$ with probability~$1-o(1/n)$. Also, the number of steps with strength~2 is $\Theta(n^2\sum_{i=b}^a1/i^2)$ in expectation. Note that we ignore the number of iterations with strengths larger than 2 for the lower bound.
    
    Overall, with probability at least~$1-o(1/n)$ the time of \sdrlss is at least
\[\mathord{\Omega}\mathord{\left((a-b)n\ln R+ n^2\sum_{i=b}^{a}\frac{1}{i^2}\right)}. \]
\end{proof}

In the following corollary, it can be seen that \sdrlsss is faster than \sdrlss in the middle of the run by a factor of roughly $n$.

\begin{corollary}\label{cor:speeduponlinearfunction}
The relative speed-up of \sdrlsss with $R\ge n^{3+\epsilon}$ for an arbitrary constant $\epsilon>0$ compared to \sdrlss with $R\ge n^{3+\epsilon^\prime}$ for an arbitrary constant $\epsilon^\prime>0$ to find a search point in $S_b$ with $b=n/8$ for a starting search point~x with $\phi(x)=n/4$ is $\Omega(n)$ with probability at least~$1-o(1/n)$.
\end{corollary}
\begin{proof}
Assume that $T_r$ and $T_m$ are the considered hitting times of \sdrlss and \sdrlsss, respectively, with $R\ge n^{3+\epsilon}$ for an arbitrary constant $\epsilon>0$.
Using Theorem~\ref{theo:linearfunctions-approx} with $a=n/4$ and $b=n/8$, we have
\[\frac{\expect{T_r}}{\expect{T_m}}\ge 
\frac{\Omega(n^2\ln R)}{O(n\ln R)}\ge \Omega(n),\]
with probability at least~$1-o(1/n)$.
\end{proof}

\section{Minimum Spanning Trees}
\label{sec:mst}
In Theorem~6 in \cite{RajabiWittEvo21}, the authors studied \sdrlss on the MST 
problem as formulated in \cite{NeumannW07} using the fitness function~$f$ that 
returns the total weight of the spanning tree and penalizes unconnected graphs as well 
as graphs containing cycles. They 
showed that \sdrlss with $R=m^4$ can find an MST 
 starting 
with an arbitrary spanning tree in $(1+o(1)) \bigl(m^2\ln m 
+ (4m\ln m)\expect{S}\bigr)$ fitness calls where $\expect{S}$ is the expected number 
of strict improvements.
The reason behind $\expect{S}$ is that for each improvement with strength~$2$, the algorithm resets the radius to one for the next epoch and explores this radius more or less completely in $m\ln R$ iterations. This can be costly for the graphs requiring many improvements.

However, with \sdrlsss, we do not need to consider the number of improvements for estimating the number of iterations with strength~$1$ since with the radius memory mechanism, the number of iterations with strength~$1$ is asymptotically in the order of $1/\ln n$ times of the number of successes. 
This leads to the following simple bound.

\begin{theorem}
\label{theo:mst}
The expected optimization time of \sdrlsss with $R=m^4$ on the MST problem starting with an arbitrary spanning tree is at most
\begin{align*}
& (1+o(1))\bigl((m^2/2)(1+\ln(r_1+\dots+r_m))\bigr) 
 \le (1+o(1)) \bigl(m^2\ln m \bigr),
\end{align*}
where $r_i$ is the rank of the $i$th edge in the sequence sorted 
by increasing edge weights.
\end{theorem}
The proof is similar to the proof of Theorem~6 in \cite{RajabiWittEvo21} by using drift multiplicative analysis \citep{LenglerDriftBookChapter}. However, we show that the radius memory mechanism controls the number of iterations with strength~$1$, and we apply Lemma~\ref{lem:realradius} to show that if the algorithm uses strengths larger than~$2$, the algorithm shortly after makes an improvement with strength~$2$ again.

\begin{proofof}{Theorem~\ref{theo:mst}}
We aim at using multiplicative drift analysis using 
$g(x)\!=\!\sum_{i=1}^m x_ir_i$ as potential function. Since the 
algorithm
has different states we do not have the same lower bound on 
the drift towards the optimum. However, at strength~$1$ no 
mutation is accepted since the fitness function 
from \cite{NeumannW07} gives a huge penalty to non-trees. Hence, 
our plan is to conduct the drift analysis conditioned 
on that the strength is at most~$2$ and account for the steps 
spent at strength~$1$ separately. Cases where the strength 
exceeds~$2$ will be handled by an error analysis and a restart argument.

Let $X^{(t)}\coloneqq g(x^{t})-g(\xopt)$ for the current 
search point~$x^{(t)}$ and an optimal search point $\xopt$. 
Since the algorithm behaves stochastically 
the same on the original fitness function~$f$ and the potential 
function~$g$, we obtain that $\expect{X^{(t)} - X^{(t+1)}\mid X^{(t)}} 
\ge X^{(t)}/\binom{m}{2}\ge 2X^{(t)}/m^2$ 
since the $g$-value can be 
decreased by altogether 
$g(x^{t})-g(\xopt)$ via a sequence of at most 
$\binom{m}{2}$ disjoint two-bit flips; see also the proof of 
Theorem~15 in \cite{DoerrJohannsenWinzenALGO12} for the underlying 
combinatorial argument. Let $T$ denote the number of steps at 
strength~$2$ 
until $g$ is minimized, assuming no larger strength to occur. 
Using the multiplicative drift theorem, 
we have $\expect{T}\le (m^2/2)(1+\ln(r_1+\dots+r_m)) \le 
(m^2/2)(1+\ln(m^2))$ and by 
the tail bounds for multiplicative drift (\eg, \cite{LenglerDriftBookChapter}) 
it holds that 
$\prob{T> (m^2/2)(\ln(m^2) + \ln(m^2))} \le e^{-\ln(m^2)} = 1/m^2$. Note 
that this bound on $T$ is below the threshold for strength~$2$ 
since $\binom{m}{2}\ln R = (m^2-m) \ln (m^4) \ge 
(m^2/2)(4\ln m)$ for $m$ large enough. Hence, with probability 
at most~$1/m^2$ the algorithm fails to find the optimum before the 
strength can change from~$2$ to a different value 
due to the threshold being exceeded. 

We next bound the expected number of steps spent at larger strengths. By Lemma~\ref{lem:realradius}, if the algorithm fails to find an improvement with the right radius, \ie when the radius becomes 
$r=3>\gap(x)=2$, then, in at most $o(\binom{m}{r-1})$ iterations in expectation, 
the radius is set to the gap size of the current search point at that time.
Thus, by using Lemma~\ref{lem:realradius}, an increase of radius above~$2$ costs at most $o\!\left(\binom{m}{2} \right)$ iterations to make improvements with strength~$2$ again and set the radius to two. This time is only a lower-order term 
of the runtime bound claimed in the theorem.
If the strength exceeds~$2$, we wait 
for it to become~$2$ again and restart the previous drift analysis, which 
is conditional on strength at most~$2$. Since the probability 
of a failure is at most~$1/m^2$, this gives an expected number 
of at most $1/(1-m^{-2})$ restarts, which is $1+o(1)$ as well.

It remains to bound the number of steps at strength~$1$. For each strict improvement,
$B$ is set to $1/\ln n\cdot u$ where $u$ is the counter value, the counter is reset, and radius~$r$ is set to~$2$. Thereafter, $B$ steps pass before
the strength becomes~$2$ again. Hence, if the strength does not exceed~$2$  
before the optimum is reached, this contributes a term of $(1+1/\ln n)$ to the number of iterations with strength~$2$ in the previous epoch. 
 Also, in the beginning of the algorithm, there is a complete phase with strength~$1$ costing $m\ln R$, which only contributes a lower-order term. 
\end{proofof}

Theorem~\ref{theo:mst} is interesting since it is asymptotically tight
and does not suffer from the additional $\log(w_{\max})$ factor
known from the analysis of the classical \oneoneea 
\citep{NeumannW07}. In fact, this seems to be the first  
asymptotically tight analysis 
 of a globally searching (1+1)-type algorithm on the MST. So far 
a tight analysis of evolutionary algorithms on the MST
was only known for \rlsonetwo with one- and two-bit flip 
mutations \citep{RaidlKJ06}. Our bound in Theorem~\ref{theo:mst} is by a factor of roughly~$2$  better 
since it avoids an expected waiting time of~$2$ for a two-bit flip. On the technical side, it is interesting
that we could apply drift analysis in its proof despite the algorithm being able to switch 
between different mutation strengths influencing the current drift.

\section{Radius Memory can be Detrimental}
\label{sec:slowdown}

After a high mutation strength has been selected, \eg, to overcome 
a local optimum, the radius memory decreases the threshold values 
for phase lengths related to 
lower strengths. As we have seen in Lemma~\ref{lem:realradius}, 
\sdrlsss can often return to a smaller strength quickly. However, 
we can also point out situations where using the smaller strength 
with their original threshold values as used in the original \sdrlss from \cite{RajabiWittEvo21} is crucial.

Our example is based on a general construction principle that 
can be traced back to \cite{WittCEC03} and was picked up in 
\cite{RajabiWittGECCO20} to show situations where stagnation in 
the context of the \ooea is detrimental; see that paper for a detailed 
account of the construction principle. In \cite{RajabiWittEvo21}, 
the idea was used to demonstrate situations where bit-flip mutations 
outperforms \sdrlss. Roughly speaking, the  
functions combine two gradients one of which is easier to exploit 
for an algorithm~$A$ while the other is easier to exploit for another algorithm~$B$.
By appropriately defining local and global optima close to the end 
of the search space in direction of the gradients, either Algorithm~$A$ 
significantly outperforms Algorithm~$B$ or the other way round.

We will now define a function on which \sdrlsss is exponentially slower 
than \sdrlss. 
In the following, we will imagine any
bit string~$x$ of length~$n$ 
as being split into a prefix $a\coloneqq a(x)$ of length~$n-m$ and a suffix $b\coloneqq b(x)$ of 
length $m$, where $m$ is defined below.
 Hence, $x=a(x)\circ b(x)$, where $\circ$ denotes the concatenation.  
The prefix~$a(x)$ is called \emph{valid} if it is of the form 
$1^i 0^{n-m-i}$, \ie, $i$ leading ones and $n-m-i$ trailing zeros. The prefix fitness 
$\pre(x)$ of a string~$x\in\{0,1\}^n$ with valid prefix 
$a(x)=1^i0^{n-m-i}$ equals~$i$, the 
number of leading ones. The suffix consists of $\lceil n^{1/8}\rceil$ consecutive blocks of 
$\lceil n^{3/4}\rceil$ bits each, altogether $m=O(n^{7/9})$ bits. 
Such a block is called \emph{valid} if it contains either~$0$ or $2$ one-bits; moreover, it is called \emph{active} if it contains~$2$ 
and \emph{inactive} if it contains~$0$ one-bits. A suffix where all blocks are valid and where all blocks following 
first inactive block are also inactive is called valid itself, and the suffix fitness $\suff(x)$ of a 
string~$x$ with valid suffix~$b(x)$  
is the number of leading active blocks before the first inactive one. Finally, we call $x\in\{0,1\}^n$ valid 
if both its prefix and suffix are valid.

The final fitness function is a weighted combination of $\pre(x)$ and $\suff(x)$. We define 
for $x\in\{0,1\}^n$, where $x=a\circ b$ with the above-introduced $a$ and~$b$, 
\begin{align*}
& \preferonebit(x)\coloneqq \\
& \begin{cases}
 n -m - \pre(x) + \suff(x)   & \text{ if $\suff(x)\le n^{1/9}$ $\wedge$  $x$ valid}\\
 n^2\pre(x) + \suff(x)& \text{ if $n^{1/9}\!<\!\suff(x)\!\le\! n^{1/8}/2$ $\wedge$ $x$ valid}\\
 n^2 (n\!-\!m) + \suff(x)-\!n-\!1  \hspace*{-3ex}& \text{ if $\suff(x)>n^{1/8}/2$ $\wedge$ $x$ valid}\\
- \onemax(x) & 
  \text{ otherwise.} 
\end{cases}
\end{align*}

We note that all search points in the third case have a fitness of at 
least~$n^2 (n-m) - n - 1$, which 
is bigger than $n^2(n-m-1) + n$, an upper bound on the fitness of search points that fall into the second case 
without having $m$ leading active blocks in the suffix. Hence, search points~$x$ where $\pre(x)=n-m$ and 
$\suff(x)=\lceil n^{1/8}\rceil$ represent local optima of second-best overall fitness. The set of global optima 
equals the points where $\suff(x)=\lfloor n^{1/8}/2\rfloor$ and $\pre(x)=n-m$, 
which implies that at least $n^{1/8}$ 
bits (two from each block) have to be flipped 
simultaneously to escape from the local toward the global optimum.  The 
first case is special in that the function is decreasing in the 
\pre-value as long as $\suff(x)\le n^{1/9}$. Typically, the 
first valid search point falls into the first case. Then two-bit flips 
are essential to make progress and the radius memory of \sdrlsss 
will be used when waiting for the next improvement. 
After leaving the first case, 
since two-bit flips happen quickly enough, the memory will 
make progress via one-bit flips unlikely, leading to the local optimum.

We note that function \preferonebit shares some features with the 
function $\needhighmut$ 
from 
\cite{RajabiWittGECCO20} and the function \needglobalmut from 
\cite{RajabiWittEvo21}. However, it  contains an extra case 
for small suffix values, uses different block sizes and block count  
for the suffix, and inverts roles of prefix and suffix 
by leading to a local optimum when the suffix is optimized first.

In the following, we make the above ideas precise and show that 
\sdrlss outperforms \sdrlsss on \preferonebit.

\begin{theorem} \label{theo:preferonebit}
With probability at least $1-1/n^{1/8}$, \sdrlsss with $R\ge n^{3+\epsilon}$ 
for an arbitrary constant $\epsilon>0$ needs time $2^{\Omega(n^{1/8})}$ 
to optimize \preferonebit. With probability at least $1-1/n$, \sdrlss 
 with $R\ge n^{3+\epsilon}$ optimizes the function in time $O(n^2)$.
 \end{theorem}
  \begin{proof}
As in the proof of Theorem~4.1 in \cite{RajabiWittGECCO20}, we have 
that the first valid search point (\ie, search point of non-negative fitness) of both \sdrlss and \sdrlsss
has both  $\pre$- and $\suff$-value 
value of at most~$n^{1/9}/2$ with probability $2^{-\Omega(n^{1/9})}$. In the following, 
we tacitly assume that we have reached a valid search point of the described maximum $\pre$- and $\suff$-value and note that 
this changes the required number of improvements to reach local or global maximum only by a $1-o(1)$ factor. For readability 
this factor will not be spelt out any more.

Given the situation with a valid search point~$x$ where
$\suff(x)<n^{1/9}/2$, fitness improvements 
only possible by increasing the \suff-value. Since the probability 
of a \suff-improving steps is at least $\binom{n^{3/2}}{2}/n^2=\Omega(n^{-1/2})$ 
at strength~$2$, the time for both algorithms to reach 
the second case of the definition of \preferonebit is $O(n^{1/9}n^{1/2})=  
O(n^{11/18})$ according to Lemma~\ref{lem:runtime-of-Em}, and 
by repeating independent phases 
and Markov's inequality, the time is $O(n)$ with 
 probability exponentially close to~$1$. Afterwards, one-bit flips 
 increasing the \pre-value are strictly improving and happen while 
 the strength is~$1$ with probability at least~$1-1/R$ 
 in \sdrlss, which does not have 
 radius memory. Therefore, by a union bound over all $O(n)$ improvements, 
 with probability at least $1-1/n$,  
 \sdrls increases the \pre-value to its maximum before the \suff-value 
 becomes greater than $\lceil n^{1/9}\rceil$. The time 
 for this is $O(n^2)$ even with probability exponentially close 
 to~$1$ by Chernoff bounds. Afterwards, by 
 reusing the above analysis to leave the first case, with probability  
 at least $1-1/n$ a number of $o(n^2)$ steps is sufficient for \sdrls 
 to find the global optimum. This proves the second statement 
 of the theorem.
 
 To prove the first statement, \ie, the claim for \sdrlsss, we first note that 
 the probability of improving the \pre-value at strength~$2$  
 is $O(n^2)$ since two specific bits would have to flip. 
 Hence, such steps never happen in $\Theta(n)$ steps 
 with probability $1-O(1/n)$. By contrast, 
 a \suff-improving step at strength~$2$, which has probability 
 $\Omega(n^{-1/2})$, happens within $O(n^{3/4})$ 
 steps with probability at least $1-1/n^{1/4}$ according to Markov's inequality. 
 In this case, the radius memory of \sdrlsss will set a threshold of 
 $B=n^{3/4}$ for the subsequent iterations at strength~$1$. The probability 
 of improving the \pre-value within this time is $O(1/n^{-1/4})$
 by a union bound, noting the success probability of at most~$1/n$. Hence, 
 the probability of having at least $n^{1/8}$ improvements of the \suff-value
 within $n^{3/4}$ steps each before an improvement of the \pre-value (at 
 strength~$1$) happens, is at least $1-1/n^{1/8}$ by a union bound. If 
 all this happens, 
 the algorithm has to flip at least $n^{1/8}$ bits simultaneously,
 which requires $2^{\Omega(n^{1/8})}$ steps already to reach the required strength. The total failure probability is $O(1/n^{1/8})$. 
 \end{proof}

\section{Experiments}
\label{sec:experiments}
We ran an implementation of five algorithms \sdrlsss, \sdrlss, \oofea with $\beta=1.5$ from \cite{DoerrLMNGECCO17}, the standard \oneoneea and \rlsonetwo on the MST problem with the 
fitness function from \cite{NeumannW07} for three types of graphs called \TG, \erdosrenyi with $p=(2\ln n)/n$, and \Kn. We carried out a similar experiment to \cite{RajabiWittEvo21} with the more additional class of complete graphs \Kn to illustrate the performance of the new algorithm and compare with the other algorithms.

\begin{figure}
    \centering
        \subfloat[\centering Graphs \TG]{{\includegraphics[width=4cm]{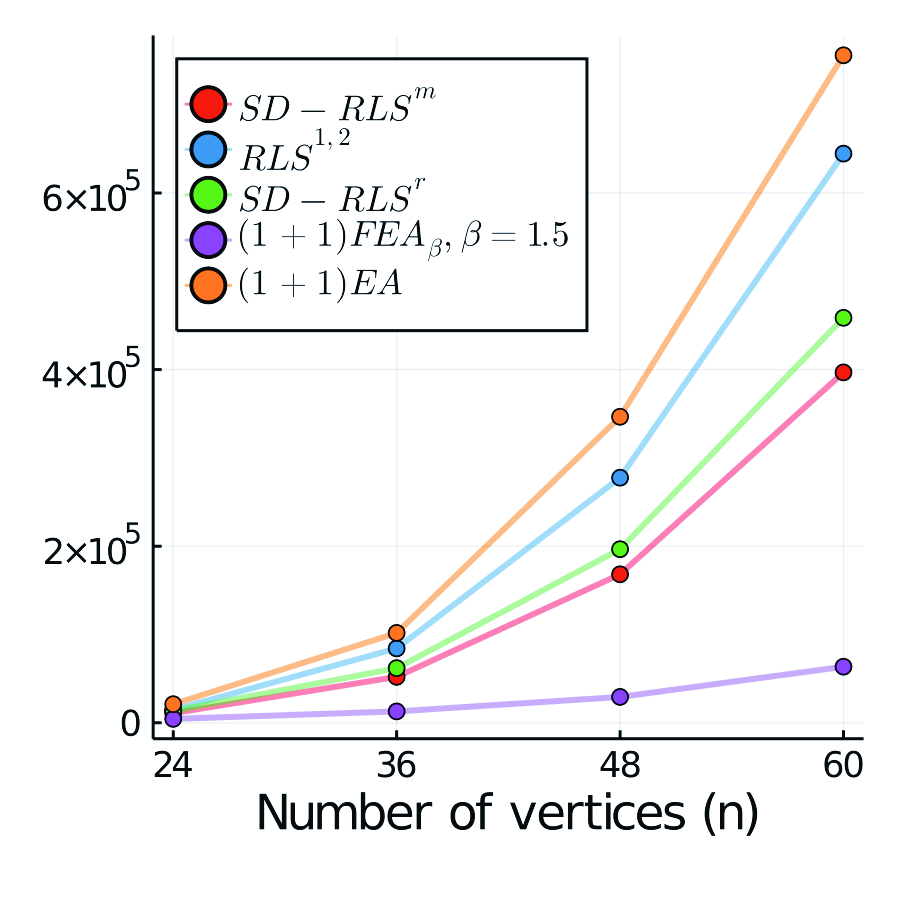} }\label{fig:TGresults}}
    \quad
    \subfloat[\centering Graphs \erdosrenyi]{{\includegraphics[width=4cm]{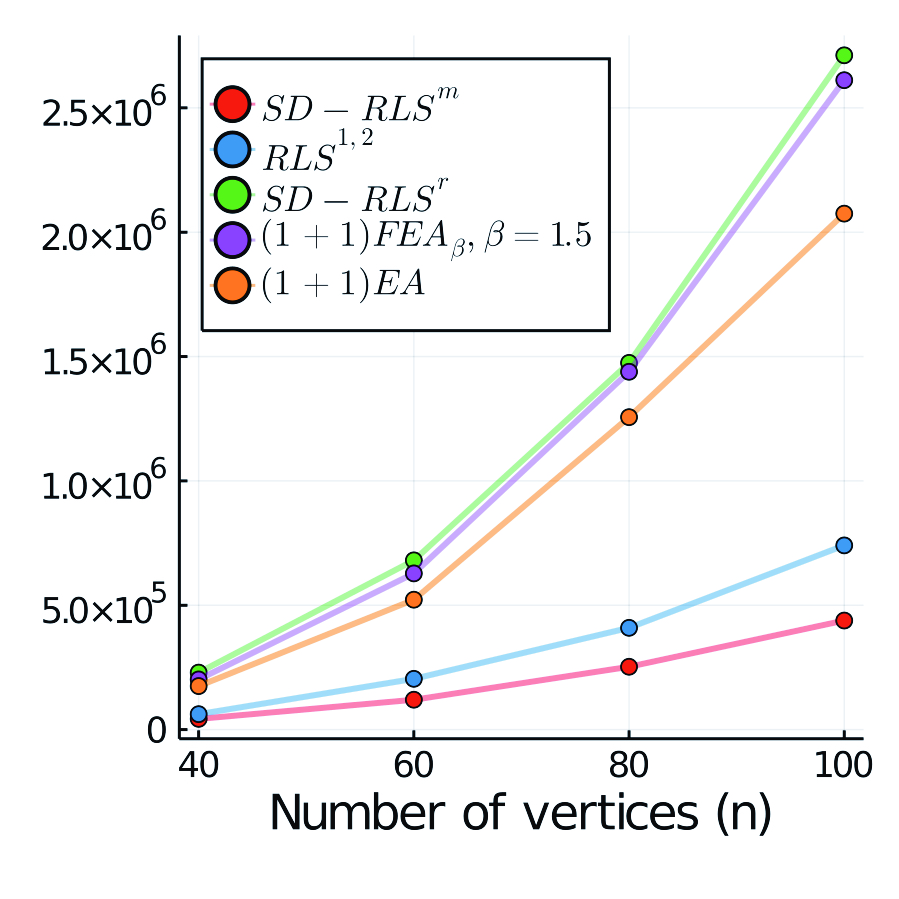} }\label{fig:ERresults}}
    \quad
    \subfloat[\centering Graphs \Kn]{{\includegraphics[width=4cm]{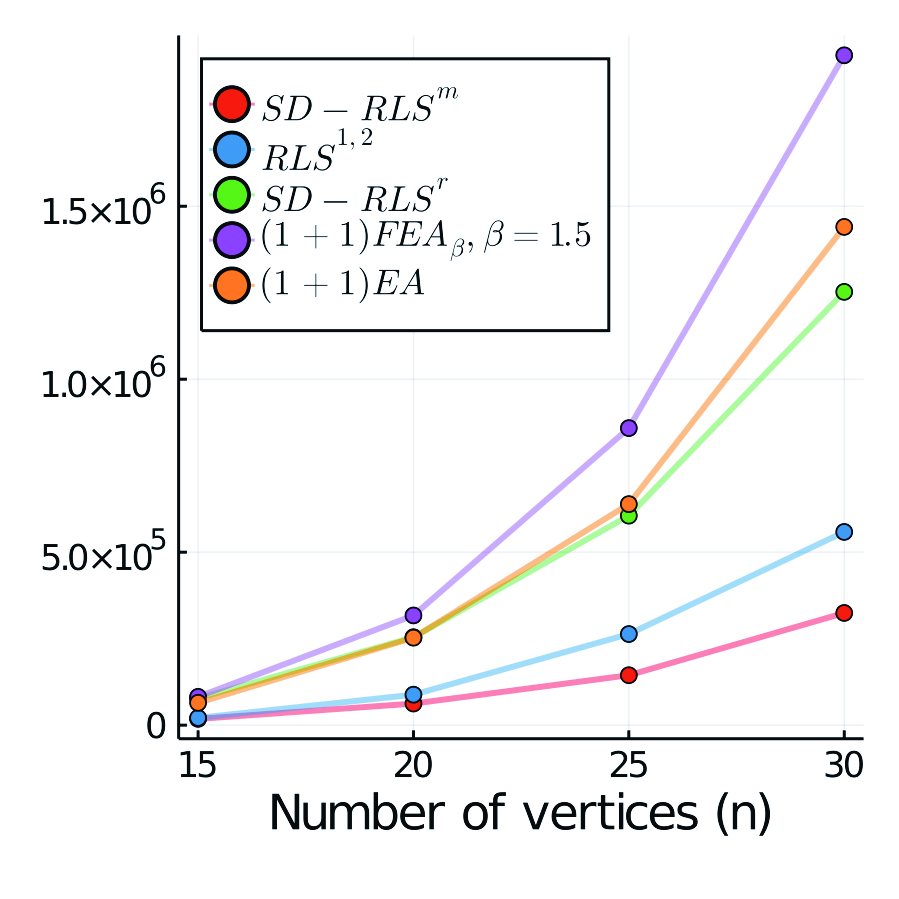} }\label{fig:Knresults}}
    \caption{Average number of fitness calls (over 200 runs) the mentioned algorithms took to optimize the fitness function MST of the graphs.}
    \label{fig:example}
\end{figure}

The graph \TG with $n$ vertices and $m=3n/4+\binom{n/2}{2}$ edges contains a sequence of $p=n/4$ triangles which are connected to each other, and the last triangle is connected to a complete graph of size $q=n/2$. Regarding the weights, the edges of the complete graph have the weight 1, and we set the weights of edges in triangle to $2a$ and $3a$ for the side edges and the main edge, respectively. In this paper, we consider $a=n^2$. The graph \TG is used for estimating lower bounds on the expected runtime of the \ooea and RLS in the literature \citep{NeumannW07}. As can be seen in Figure~\ref{fig:TGresults}, \ooea with heavy-tailed mutation, \oofea, with $\beta=1.5$ outperformed the rest of the algorithms. However,  \sdrlss and \sdrlsss also outperformed the standard \ooea and \rlsonetwo.

Regarding the graphs \erdosrenyi, we produced some random \erdosrenyi graphs with $p=(2\ln n)/n$ and assigned each edge an integer weight in the range $[1,n^2]$ uniformly at random. We also checked that the graphs are certainly connected. Then, we ran the implementation to find the MST of these graphs. The obtained results can be seen in Figure~\ref{fig:ERresults}. As discussed in Section~6 in \cite{RajabiWittEvo21}, \sdrlss does not outperform the \ooea and \rlsonetwo on MST with graphs when the number of strict improvements in \sdrlss is large. However, the proposed algorithm in this paper, \sdrlsss outperformed the rest of the algorithms, although there can be a relatively large number of improvements on such graphs. We can also see this superiority in 
Figure~\ref{fig:Knresults} for the complete graphs \Kn with random edge weights in the range $[1,n^2]$.

For statistical tests, we ran the algorithms on the graphs \TG and \erdosrenyi 200 times, and all p-values obtained from a Mann-Whitney U-test between the algorithms, with respect to the null hypothesis of identical behavior, are less than $10^{-2}$ except for the results regarding the smallest size in each set of graphs.

\section*{Conclusions}
We have investigated stagnation detection with the $s$-bit flip
operator as known from randomized local search and introduced 
a mechanism called \emph{radius memory} that allows continued 
exploitation 
of large~$s$ values that were useful in the past. Improving earlier
work from \cite{RajabiWittEvo21}, this leads to tight bounds 
on complex 
multimodal problems like linear functions 
with uniform constraints and the minimum spanning tree problem, 
while still optimizing unimodal and jump functions as efficiently as in 
earlier work. 
The bound for the MST is the first tight runtime bound for a global
search heuristics and improves upon the runtime of classical RLS 
algorithms by a factor of roughly~$2$. 
We have also pointed out situations where 
the radius memory is detrimental for the optimization process. 
In the future,
we would like to investigate the concept of stagnation detection 
with radius memory in population-based algorithms and plan analyses on
further combinatorial optimization problems.

 \section*{Acknowledgement}
	This work was supported  by a grant by the Danish Council for Independent Research  (DFF-FNU  8021-00260B).

\bibliographystyle{mynatbib_english}
\bibliography{references}

\end{document}